\documentclass{article}

\usepackage{hyperref}
\usepackage{amsmath}
\usepackage{amssymb}
\usepackage{amsthm}
\usepackage{tikz}
\usepackage{pgfplots}
\usepackage{caption}

\pgfplotsset{compat=1.4} 

\newcounter{global}
\theoremstyle{definition}

\theoremstyle{plain}
\newtheorem{theorem}[global]{Theorem}

\newtheorem{lemma}[global]{Lemma}
\newtheorem{corollary}[global]{Corollary}
\newtheoremstyle{note}{}{}{}{}{\itshape}{.}{.5em}{}
\theoremstyle{note}
\newtheorem{remark}{Remark}%
\newtheorem{example}{Example}%
\renewcommand\section{%
  \@startsection {section}{1}{\z@}%
  {-3.5ex \@plus -1ex \@minus -.2ex}%
  {2.3ex \@plus.2ex}%
  {\normalfont\large\bfseries}}
\def\itm#1{{\rm(\textit{\romannumeral#1})}}

\newcommand{\yields}[3][\infty]{\ensuremath{#2 \mathrel{\sqsupset^*_{#1}} #3}}
\begin{document}

\title{Computing sets of graded attribute implications with witnessed non-redundancy}

\date{\normalsize%
  Dept. Computer Science, Palacky University Olomouc}

\author{Vilem Vychodil\footnote{%
    e-mail: \texttt{vychodil@binghamton.edu},
    phone: +420 585 634 705,
    fax: +420 585 411 643}}

\maketitle

\begin{abstract}
  In this paper we extend our previous results on sets
  of graded attribute implications with witnessed non-redundancy. We assume
  finite residuated lattices as structures of truth degrees and use arbitrary
  idempotent truth-stressing linguistic hedges as parameters which influence
  the semantics of graded attribute implications. In this setting,
  we introduce algorithm which transforms any set of graded attribute
  implications into an equivalent non-redundant set of graded attribute
  implications with saturated consequents whose non-redundancy is witnessed
  by antecedents of the formulas. As a consequence, we solve the open
  problem regarding the existence of general systems of pseudo-intents
  which appear in formal concept analysis of object-attribute data with
  graded attributes and linguistic hedges. Furthermore,
  we show a polynomial-time procedure for determining bases given by
  general systems of pseudo-intents from sets of graded attribute
  implications which are complete in data.
\end{abstract}

\section{Introduction}
In this paper, we investigate properties of sets of graded attribute
implications and extend the results presented in~our recent
paper~\cite{Vy:Osgaiwnr}. The graded attribute implications,
sometimes called fuzzy attribute implications~\cite{BeChVy:Ifdwfa},
are rules describing if-then dependencies in data with
graded attributes. The rules have
been proposed and investigated from the point of view of formal
concept analysis (shortly, FCA~\cite{GaWi:FCA}) with
linguistic hedges~\cite{BeVy:Fcalh}. One of the basic problems in FCA is
to extract, given a formal context, a set of attribute implications
which is non-redundant and conveys the information about exactly all
attribute implications which hold in the given formal context---such
sets are called (non-redundant) bases of formal contexts. One of the most
profound approaches of determining bases exploits the notion of
a pseudo-intent which originated in~\cite{GuDu} and has been
later utilized, e.g., in~\cite{Ga:Tbaca}. The bases given by pseudo-intents
are not only non-redundant but in addition minimal in terms of their cardinality.
In our paper, we deal with a general notion of a system of pseudo-intents
which appears in the generalization of FCA which includes graded attributes
and uses linguistic hedges to reduce the size of
concept lattices~\cite{BeVy:Rsclh}.
By a graded attribute we mean an attribute (property/feature) which
may be satisfied (present) to degrees instead of just
satisfied/not satisfied (present/not present) as in the ordinary setting.
In the past, there have been many approaches to extensions of the traditional
concept analysis which accommodate graded
attributes~\cite{Bel:FRS,Krs:Cbefc,MeOjRu:Fcavmacl,Po:FB}
and related phenomena. Most of the approaches are focused solely
on the structure of concept lattices with little or no attention
paid to if-then rules. The exceptions seem the be the early works
by Polland~\cite{Po:FB} and the results made in the framework of
FCA with linguistic hedges, see~\cite{BeVy:ADfDwG} for a survey.
In~\cite{Po:FB}, the author proposes generalized pseudo-intents which
ensure that the constructed sets of formulas are complete in data,
i.e., convey the information
about exactly all if-then rules which hold in the data,
but are redundant in general. Using a more general setting,
\cite{BeChVy:Ifdwfa,BeVy:Falaitvenb} show that there is a general
notion of a system of pseudo-intents which ensures both the completeness
and non-redundancy.
Unfortunately, the definition in \cite{BeChVy:Ifdwfa} is not constructive and
so far the procedure to find such systems was reduced to finding particular
maximal independent sets of vertices in
large graphs~\cite{BeVy:Faicnbumis,BeVy:Cnbirdtga}.
In addition, it has been shown that the existence and
uniqueness of systems of pseudo-intents is not ensured in the general
setting. Indeed, it follows that the properties of the underlying
structures of truth degrees, which together with linguistic hedges
determine the semantics of graded attribute implications,
substantially affect the properties of such systems. In case of infinite
structures of truth degrees, it is known that general systems of pseudo-intents
may not exist~\cite{BeVy:ADfDwG}. The existence in case of finite structures
was listed as one of the open problems in~\cite{Kw:Open2006}.
Our paper brings a positive answer to this question and shows that,
among other results, that general systems of pseudo-intents can
be determined in a polynomial time from any complete set of graded
attribute implications. The result is based on some of our recent
observations made in~\cite{Vy:Osgaiwnr} where we have put in correspondence bases
given by systems of pseudo-intents and non-redundant sets of graded
attribute implications with saturated consequents where the non-redundancy
of each formula is witnessed by its antecedent.

Detailed description of the problem and the results requires precise
introduction of the utilized notions. Therefore, we postpone it
to Section~\ref{sec:problem} after presenting the preliminaries
in Section~\ref{sec:prelim}. In Section~\ref{sec:problem}, we include
the algorithm and comment on its immediate consequences. The soundness
of the algorithm is proved in Section~\ref{sec:proof} which also
contains additional remarks and examples.
Finally, we present conclusions in Section~\ref{sec:concl}.

\section{Preliminaries}\label{sec:prelim}
In this section, we present the basic notions related to the structures of
truth degrees which are used in our paper and recall basic notions of graded
attribute implications. We limit ourselves just to the notions which are
utilized in this paper. Interested readers can find more
details in~\cite{BeVy:ADfDwG}.
Readers familiar with~\cite{Vy:Osgaiwnr} can skip this section and go
directly to Section~\ref{sec:problem}.

A residuated lattice~\cite{Bel:FRS,GaJiKoOn:RL}
is an algebra
$\mathbf{L}=\langle L,\wedge,\vee,\otimes,\rightarrow,0,1\rangle$
where $\langle L,\wedge,\vee,0,1 \rangle$ is a bounded lattice with $0$
and $1$ being the least and the greatest elements of $L$, respectively,
$\langle L,\otimes,1 \rangle$ is a commutative monoid (i.e., $\otimes$ is
commutative, associative, and $1$ is neutral with respect to~$\otimes$),
and $\otimes$ and $\rightarrow$ satisfy the so-called
adjointness property: for all $a,b,c \in L$,
we have that $a \otimes b \leq c$ if{}f $a \leq b \rightarrow c$.
Further in the paper, $\mathbf{L}$ always stands for a residuated lattice
of the form $\mathbf{L}=\langle L,\wedge,\vee,\otimes,\rightarrow,0,1\rangle$.
$\mathbf{L}$ is a complete residuated lattice whenever 
$\langle L,\wedge,\vee,0,1 \rangle$ is a complete lattice,
(i.e., infima and suprema exist for arbitrary subsets of~$L$).
If $L$ is finite, then $\mathbf{L}$ is trivially complete.
Examples of (complete) residuated lattices include popular structures
defined on the real unit interval using left-continuous triangular
norms~\cite{EsGo:MTL,KMP:TN} and their finite substructures. The
structures are utilized in mathematical fuzzy
logics~\cite{CiHa:Tnbpfl,Gog:Lic,Got:Mfl,Haj:MFL} and
their applications~\cite{KlYu} as structures of truth degrees
with $\otimes$ and $\rightarrow$ used as truth functions of
``fuzzy conjunction'' and ``fuzzy implication'', respectively.

As usual, a map $A\!: Y \to L$ is called an $\mathbf{L}$-set $A$ in $Y$
(or an $\mathbf{L}$-fuzzy set~\cite{Gog:LFS}); $R\!: X \times Y \to L$ is
called a binary $\mathbf{L}$-relation between $X$ and $Y$,
$R(x,y) \in L$ is interpreted as the degree to which $x \in X$
and $y \in Y$ are related by $R$.
The collection of all $\mathbf{L}$-sets in $Y$ is denoted by $L^Y$.
Operations with $\mathbf{L}$-sets are defined componentwise using
operations in $\mathbf{L}$. For instance,
the union $A \cup B$ of $\mathbf{L}$-sets $A \in L^Y$
and $B \in L^Y$ is an $\mathbf{L}$-set in $Y$
such that $(A \cup B)(y) = A(y) \vee B(y)$;
analogously for $\cap$ and $\wedge$.
If $a \in L$ and $A \in L^Y$ then $a{\otimes}A$,
called the $a$-multiple of $A$, is an $\mathbf{L}$-set in $Y$
defined by $(a{\otimes}A)(y) = a \otimes A(y)$ for all $y \in Y$.
For $A,B \in L^Y$, we define the degree $S(A,B)$ to which $A$ is
a subset of $B$ by
\begin{align}
  S(A,B) = \textstyle{\bigwedge}_{y \in Y}\bigl(A(y) \rightarrow B(y)\bigr)
  \label{eqn:S}
\end{align}
provided that the infimum of
$\{A(y) \rightarrow B(y);\, y \in Y\} \subseteq L$ exists---this condition
is satisfied, e.g., if $\mathbf{L}$ is complete or if $Y$ is finite.
Note that $S$ given by~\eqref{eqn:S} can be understood as a binary
$\mathbf{L}$-relation on $\mathbf{L}$-sets, i.e., for a fixed $Y$,
it is a map of the form $S\!: L^Y \times L^Y \to L$. It is easily
seen that $S(A,B) = 1$ if{}f $A(y) \leq B(y)$ for all $y \in Y$
in which case we write $A \subseteq B$ and say that
$A$ is a full subset of $B$.

\begin{remark}
  Let us note that the notion of a graded subsethood~\eqref{eqn:S} defined
  using the residuated implication has been proposed by
  Goguen~\cite{Gog:LFS,Gog:Lic} and plays an important role in
  the interpretation of the if-then rules we consider in this paper.
  This corresponds with the fact that the usual inclusion of sets
  is used to define the interpretation of the classic attribute implications.
  Indeed, if $Y$ is a non-empty set of attributes (symbolic names),
  any formula of the form $A \Rightarrow B$ where $A,B \subseteq Y$
  is called an attribute implication~\cite{GaWi:FCA}. Moreover,
  it is considered true given $M \subseteq Y$,
  written $M \models A \Rightarrow B$,
  whenever $A \subseteq M$ implies $B \subseteq M$. If we depart from
  the classic setting to the graded setting and replace the classic sets
  by $\mathbf{L}$-sets $A,B,M \in L^Y$, there are several possible ways
  how to define the notion of ``$A \Rightarrow B$ being true in $M$''
  which all collapse to the ordinary notion when $\mathbf{L}$ is
  the two-element Boolean algebra. As it is described in detail
  in~\cite{BeVy:ADfDwG}, two borderline (and both interesting) cases
  can be based on the graded and the full inclusion of $\mathbf{L}$-sets.
\end{remark}

The framework we use in our paper enables us to reason with several
different interpretations of inclusion of $\mathbf{L}$-sets, and thus
several different ways of understanding the interpretation of 
data dependencies, using a single formalism which is based on additional
parameterization of the semantics of the rules. Namely, we use
the approach based on linguistic hedges~\cite{BeVy:ADfDwG}.
In a more detail, given a non-empty and finite set $Y$ of
attributes and a complete
residuated lattice $\mathbf{L}$, a~graded attribute implication in $Y$
is an expression $A \Rightarrow B$ where $A,B \in L^Y$; $A$ is called
the antecedent of $A \Rightarrow B$, $B$ is called the consequent
of $A \Rightarrow B$. Furthermore, let ${}^*$ be a map ${}^*\!: L \to L$.
For $A,B,M \in L^Y$, the degree 
$||A \Rightarrow B||^*_M$ to which $A \Rightarrow B$ is true in $M$
(under $\mathbf{L}$ and ${}^*$) is defined by
\begin{align}
  ||A \Rightarrow B||^*_M &= S(A,M)^* \rightarrow S(B,M).
  \label{eqn:fai_truth}
\end{align}
In particular, ${}^*$ used in our considerations is always
an idempotent truth-stressing linguistic hedge~\cite{BeVy:ADfDwG},
i.e., it satisfies the following conditions:
\begin{align}
  1^* &= 1, \label{TS:1} \\
  a^* &\leq a, \label{TS:subd} \\
  (a \rightarrow b)^* &\leq a^* \rightarrow b^*, \label{TS:impl} \\
  a^{**} &= a^* \label{TS:idm}
\end{align}
for all $a,b \in L$. Now, two important cases of~\eqref{eqn:fai_truth} result
by setting ${}^*$ to the identity (i.e., ${}^*$ is a map such that $a^* = a$
for all $a \in L$) and the so-called globalization~\cite{TaTi:Gist} in which
case $a^* = 1$ for $a = 1$ and $a^* = 0$ for $a < 1$. We refrain from detailed
description of the role of hedges because it is presented elsewhere,
see~\cite{Za:Afstilh,Za:lv1,Za:lv2,Za:lv3} for the initial development of
hedges in fuzzy logic in the wide sense,
\cite{Haj:Ovt,EsGoNo:Hedges}
for the treatment in the context of mathematical fuzzy
logics, and~\cite{BeVy:ADfDwG} for the explanation of the role of hedges
as parameters of the interpretation of if-then rules. Also note that the
approach to parameterization of the interpretation of if-then rules by hedges
can be seen as a particular case of a more general approach which utilizes
systems of isotone Galois connections and is described in~\cite{Vy:Psfaisigc}.

In this paper, we utilize the notion of the semantic entailment of graded
attribute implications. Each set of graded attribute
implications (in a fixed $Y$) is called a theory and is denoted by
$\Sigma,\Gamma,\Delta,\ldots$
An $\mathbf{L}$-set $M \in L^Y$ is called a model of $\Sigma$ whenever
$||A \Rightarrow B||^*_M = 1$ for all $A \Rightarrow B \in \Sigma$.
The set of all models of $\Sigma$ is denoted by $\mathrm{Mod}^*(\Sigma)$.
The degree $||A \Rightarrow B||^*_\Sigma$ to
which $A \Rightarrow B$ is semantically entailed by $\Sigma$ is defined by
\begin{align}
  ||A \Rightarrow B||^*_\Sigma &=
  \textstyle\bigwedge_{M \in \mathrm{Mod}^*(\Sigma)}||A \Rightarrow B||^*_M.
  \label{eqn:sement}
\end{align}
We put $\Sigma \mathrel{\unlhd} \Gamma$ and say that $\Gamma$
is stronger than $\Sigma$ whenever
\begin{align}
  ||A \Rightarrow B||^*_\Sigma \leq ||A \Rightarrow B||^*_\Gamma
\end{align}
for all $A,B \in L^Y$. Furthermore, $\Sigma$ and $\Gamma$ are equivalent,
written $\Sigma \equiv \Gamma$, whenever $\Sigma \mathrel{\unlhd} \Gamma$ and
$\Gamma \mathrel{\unlhd} \Sigma$.
It is a well-known fact that $\Gamma \equiv \Sigma$
if{}f $\mathrm{Mod}^*(\Sigma) = \mathrm{Mod}^*(\Gamma)$,
see~\cite{BeVy:ICFCA,BeVy:ADfDwG}.
Theory $\Sigma$ is called redundant whenever there is $\Gamma \subset \Sigma$
such that $\Sigma \equiv \Gamma$; otherwise $\Sigma$ is called non-redundant.
Theory $\Sigma$ is called minimal~\cite{Vy:Omsgai} whenever for
each $\Gamma$ such that $\Sigma \equiv \Gamma$, we have $|\Sigma| \leq |\Gamma|$.
One of the basic properties of non-redundant sets of graded attribute
implications is the following: $\Sigma$ is non-redundant if{}f
$||A \Rightarrow B||^*_{\Sigma \setminus \{A \Rightarrow B\}} < 1$ for
all $A \Rightarrow B \in \Sigma$, see~\cite{BeVy:Pmfai}.
Furthermore,
if $||A \Rightarrow B||^*_{\Sigma \setminus \{A \Rightarrow B\}} = 1$
for $A \Rightarrow B \in \Sigma$, we say that $A \Rightarrow B$
is redundant in $\Sigma$.

Recall from~\cite{Vy:Osgaiwnr} that if $\Sigma$ is non-redundant,
we say that the non-redundancy of $\Sigma$ is witnessed
(by the antecedents of the formulas in $\Sigma$)
whenever for every $A \Rightarrow B \in \Sigma$, we have
that $A \in \mathrm{Mod}^*(\Sigma \setminus \{A \Rightarrow B\})$.
This is one of the key notions used in the present paper since
we seek a procedure to transform a given set of graded attribute
implications into an equivalent one which is non-redundant and its
non-redundancy is witnessed.

Finally, we have to recall the notion of least models which is useful
in the characterization of the semantic entailment. Recall that
each $\mathrm{Mod}^*(\Sigma)$ is an $\mathbf{L}^{\!*}$-closure
system~\cite{BeFuVy:Fcots,BeVy:Pmfai}. As a consequence, we
may introduce the least model $[M]^*_\Sigma$ of $\Sigma$
which contains $M$:
\begin{align}
  [M]^*_\Sigma
  &= 
  \textstyle\bigcap\{N \in \mathrm{Mod}^*(\Sigma);\, M \subseteq N\}.
  \label{eqn:least_mod}
\end{align}
According to~\cite[Theorem 3.11]{BeVy:ADfDwG},
$||A \Rightarrow B||^*_\Sigma = S(B,[A]^*_\Sigma)$ for any $\Sigma$
and all $A,B \in L^Y$. Furthermore, $[{\cdots}]^*_\Sigma$ is an 
$\mathbf{L}^{\!*}$-closure operator~\cite{BeFuVy:Fcots}, i.e.,
it satisfies the following conditions
\begin{align}
  A &\subseteq [A]^*_\Sigma, \label{eqn:cl1} \\
  S(A,B)^* &\leq S([A]^*_\Sigma, [B]^*_\Sigma), \label{eqn:cl2} \\
  [[A]^*_\Sigma]^*_\Sigma &= [A]^*_\Sigma, \label{eqn:cl3}
\end{align}
for all $A,B \in L^Y$. Note that~\eqref{eqn:cl1} and~\eqref{eqn:cl3} are
the usual extensivity and idempotency; \eqref{eqn:cl2} is a stronger form
of monotony and, in particular, \eqref{eqn:cl2} implies that
$[A]^*_\Sigma \subseteq [B]^*_\Sigma$ provided that $A \subseteq B$.
Using least models, we can introduce graded attribute implications
whose consequents are largest with respect to a given theory as follows:
we say that $A \Rightarrow B \in \Sigma$ has a saturated consequent
(with respect to $\Sigma$) whenever $B = [A]^*_\Sigma$.
A theory where each formula has a saturated consequent is called
a theory with saturated consequents.
According to~\cite[Lemma~4]{Vy:Osgaiwnr}, for each $\Gamma$ and
\begin{align}
  \Sigma = \{A \Rightarrow [A]^*_\Gamma;\, A \Rightarrow B \in \Gamma\},
  \label{eqn:satur}
\end{align}
we have $\Gamma \equiv \Sigma$. Therefore, for each theory there is
an equivalent theory consisting of graded attribute implications with
saturated consequents.

\section{Problem setting and results}\label{sec:problem}
In our previous paper~\cite{Vy:Osgaiwnr}, we have shown that considering
the globalization as the hedge, each theory can be transformed into
a theory which is equivalent, non-redundant, and its non-redundancy
is witnessed by antecedents of the formulas in the theory. Namely,
if ${}^\bullet$ denotes globalization on $\mathbf{L}$, for any theory $\Delta$,
we can consider 
\begin{align}
  \Gamma \subseteq \{A \Rightarrow [A]^\bullet_\Delta;\,
  A \Rightarrow B \in \Delta\},
\end{align}
such that $\Gamma$ is non-redundant and $\Gamma \equiv \Delta$,
see~\cite[Lemma 4]{Vy:Osgaiwnr}.
Furthermore, \cite[Theorem 9]{Vy:Osgaiwnr} yields that
\begin{align}
  \Sigma &=
  \bigl\{[A]^\bullet_{\Gamma \setminus \{A \Rightarrow [A]^\bullet_\Gamma\}}
  \Rightarrow
  [A]^\bullet_\Gamma;\, A \Rightarrow [A]^\bullet_\Gamma \in \Gamma\bigr\}.
  \label{eqn:Sigma_infl}
\end{align}
is non-redundant, its non-redundancy is witnessed,
and $\Sigma \equiv \Gamma \equiv \Delta$, i.e., $\Sigma$ is the desired
theory. Since $\Delta \equiv \Sigma$, $\Sigma$ can be interpreted
as a theory which conveys the same information as $\Delta$.
As a consequence of having ${}^\bullet$ as the globalization,
$\Sigma$ is minimal in terms of the number of formulas it contains and
may be significantly smaller than $\Delta$. In addition, $\Sigma$ may be
considered more informative than $\Delta$ because (i) its consequents
are saturated, i.e., they are the largest possible $\mathbf{L}$-sets
of attributes,
and (ii) its non-redundancy is witnessed---the non-redundancy of each formula
in $\Sigma$ is directly observed from its antecedent because it is a model
of the remaining formulas in $\Sigma$. Therefore, computing $\Sigma$
from $\Delta$ is desirable from several points of view.

\begin{remark}
  Note for readers familiar with non-redundant bases generated from data:
  \cite{Vy:Osgaiwnr} shows that for all equivalent theories with
  saturated consequents, $\Sigma$ given by \eqref{eqn:Sigma_infl} is
  uniquely given (considering ${}^\bullet$ as
  the globalization) and the set of its antecedents forms a system of
  pseudo-intents of any formal $\mathbf{L}$-context whose intents are exactly
  the models of $\Sigma$. As a consequence, if one finds a set of graded
  attribute implications which is complete in a given formal
  $\mathbf{L}$-context considering globalization as the linguistic hedge,
  then the system of pseudo-intents of the formal $\mathbf{L}$-context
  can be computed from the complete set of graded attribute implications
  in a polynomial time. Indeed, as we have outlined here and as it is
  presented in~\cite{Vy:Osgaiwnr}, the procedure involves computations of
  least models for each formula in the complete set which may be done in
  polynomial time and efficient algorithms for computation of closures
  exist~\cite{BeCoEnMoVy:Aermdsod}.
  This is in contrast with computing the systems of pseudo-intents directly
  from a formal $\mathbf{L}$-context which is known to be hard even in
  the classic case~\cite{DiSe:OCEP}.
\end{remark}

The goal of this paper is to show that equivalent non-redundant theories
with saturated consequents and witnessed non-redundancy can be computed for
any theory considering an arbitrary idempotent truth-stressing hedge and
finite $\mathbf{L}$, thus extending the previous results in~\cite{Vy:Osgaiwnr}
which was presented only for globalization. In such a setting, the semantics of
graded attribute implications is very specific and all truth degrees that
appear in antecedents and consequents of formulas can be seen as hard
thresholds. Indeed, from~\eqref{eqn:fai_truth} and the fact that ${}^\bullet$
is globalization, it follows that
\begin{align}
  ||A \Rightarrow B||^\bullet_M &=
  \left\{
    \begin{array}{@{\,}l@{\quad}l@{}}
      1, &\text{if } S(A,M) < 1 \text{ (i.e., $A \nsubseteq M$),}
      \\
      S(B,M), &\text{otherwise.}
    \end{array}
  \right.
\end{align}
In particular, $||A \Rightarrow B||^\bullet_M = 1$ means $A \subseteq M$
implies $B \subseteq M$. In words, if each $y \in Y$ is present in $M$
at least to the degree to which it is present in $A$,
then each $y \in Y$ is present in $M$ at least to the degree to which
it is present in $B$. Therefore, the degrees $A(y)$ and $B(y)$ can be
seen as ``hard thresholds'' for the presence of $y \in Y$. In contrast,
when ${}^*$ is a hedge other that the globalization, the thresholds are
not so strict and may be seen as ``soft thresholds''. This is best seen
in the case of ${}^*$ being identity where $||A \Rightarrow B||^*_M = 1$
means that the degree to which $B$ is included in $M$ is at least as
high as the degree to which $A$ is included in $M$. In general,
we have $||A \Rightarrow B||^*_M \leq ||A \Rightarrow B||^\bullet_M$
but not \emph{vice versa.} For some structures of
truth degrees, it seems that non-redundant bases computed from data using
general ${}^*$ tend to be smaller than bases computed using ${}^\bullet$
as it follows from the preliminary experimental evaluation
we include in Section~\ref{sec:proof}.
Thus, in addition to solving the problem of existence of general systems
of pseudo-intents, this may be seen as a practical motivation because
users usually want to infer bases as small as possible.

Let us also recall that in~\cite[Example~4]{Vy:Osgaiwnr},
we have presented observations showing
that the procedure of obtaining $\Sigma$ as in~\eqref{eqn:Sigma_infl}
is not directly applicable in case of general theories and general hedges
because $\Sigma$ may not be equivalent to $\Gamma$ and the initial
theory $\Delta$. In this paper, we introduce an alternative approach
which is conceptually close to that in~\cite{Vy:Osgaiwnr}
but always ensures that
the results are equivalent to the initial theories and have all the
desired properties. The approach is based on a procedure which
sequentially modifies the initial theory and in finitely many steps
produces the result. An elementary step of the transformation of
theories is formalized by a binary relation $\yields[]{}{}$ on the
set of all theories and it is defined as follows:
For theories $\Delta$ and $\Gamma$, we put
\begin{align}
  \yields[]{\Delta}{\Gamma}
  \label{eqn:reduces}
\end{align}
whenever $\Delta \ne \Gamma$ and there are two distinct
formulas $A \Rightarrow B \in \Delta$
and $C \Rightarrow D \in \Delta$ such that
\begin{align}
  \Gamma = \bigl(\Delta \setminus \{A \Rightarrow B\}\bigr)
  \cup \{A \cup (S(C,A)^* \otimes D) \Rightarrow B\}.
  \label{eqn:reduct_Gamma}
\end{align}
Furthermore, $\Delta$ is called irreducible (under ${}^*$) whenever
there is no $\Gamma$ such that $\yields[]{\Delta}{\Gamma}$. By definition,
$\yields[]{}{}$ is a relation on the set of all theories.
The fact $\yields[]{\Delta}{\Gamma}$ can be read
``$\Delta$ reduces to $\Gamma$ (in a single step)'' and it represents
an elementary step in a procedure which converts given theory to
a theory with the desired properties. For convenience, we introduce
the reflexive and transitive closure of $\yields[]{}{}$ in the usual way:
We put $\yields[0]{\Delta}{\Delta}$ for any $\Delta$ and,
for any natural~$n$, we put 
$\yields[n]{\Delta}{\Gamma}$ whenever there is some $\Xi$ such that
$\yields[n-1]{\Delta}{\Xi}$ and $\yields[]{\Xi}{\Gamma}$.
Finally, $\yields{}{}$ is the union of all $\yields[n]{}{}
$ for all non-negative $n$.

\begin{theorem}\label{th:witnessed}
  Let $\mathbf{L}$ be finite, ${}^*$ be any idempotent truth-stressing hedge,
  $\Gamma$ be a non-redundant/minimal set of graded attribute implications
  in $Y$ with saturated consequents. Then, for each irreducible $\Sigma$
  such that $\yields{\Gamma}{\Sigma}$, the following conditions are satisfied:
  \begin{enumerate}
  \item[\itm{1}]
    $\Sigma \equiv \Gamma$,
  \item[\itm{2}]
    $\Sigma$ is non-redundant/minimal,
  \item[\itm{3}]
    the non-redundancy of $\Sigma$ is witnessed.
  \end{enumerate}
\end{theorem}

Before we elaborate the proof in Section~\ref{sec:proof}, let us comment on
some immediate consequences of Theorem~\ref{th:witnessed}. The theorem implies
that we can always start with any $\Gamma$ and convert it, in finitely
many steps, into $\Sigma$ with properties~\itm{1}--\itm{3}. This is
a consequence of two facts: First, using~\eqref{eqn:satur},
for any $\Gamma$ there is an equivalent and non-redundant set $\Gamma'$
with saturated consequents
which can be computed in finitely many steps just by a repeated computation
of least models in order to saturate the consequents of formulas in $\Gamma$
and to remove redundant formulas. Clearly, such a procedure is polynomial.
Second, we can apply Theorem~\ref{th:witnessed} with $\Gamma'$ to get the
desired $\Sigma$---it remains to prove that an irreducible $\Sigma$
such that $\yields{\Gamma'}{\Sigma}$ and \itm{1}--\itm{3} hold can be found
in finitely many steps but this fact comes almost immediately. Indeed,
observe that $\yields[]{\Delta}{\Xi}$ means that $\Xi$ differs
from $\Delta$ in a single formula, namely in its antecedent, which is
larger (in sense of the inclusion $\subseteq$ of $\mathbf{L}$-sets).
Since both $\mathbf{L}$ and $Y$ are assumed to be finite,
reduction of any theory always terminates by an irreducible element
and the number of elementary steps is, again, polynomial in the size of
the initial theory.
Therefore, it remains to prove that the outcome of the reduction is indeed
a theory with the desired properties---this is shown in
Section~\ref{sec:proof}.

An important consequence of Theorem~\ref{th:witnessed} is for the existence
of general systems of pseudo-intents of object-attribute data with graded
attributes.
Let us recall from~\cite{Vy:Osgaiwnr} that for non-empty finite sets $X$
(set of objects) and $Y$ (set of attributes), and a binary $\mathbf{L}$-relation
$I\!: X \times Y \to L$, the triplet $\mathbf{I} = \langle X,Y,I\rangle$ is
called a formal $\mathbf{L}$-context~\cite{Bel:FRS}.
The degree $||A \Rightarrow B||^*_\mathbf{I}$
to which $A \Rightarrow B$ ($A,B \in L^Y$) is true
in $\mathbf{I}$ (under ${}^*$), see~\cite{BeChVy:Ifdwfa}, is defined as
\begin{align}
  ||A \Rightarrow B||^*_\mathbf{I} &=
  \textstyle\bigwedge_{x \in X} ||A \Rightarrow B||^*_{I_x}
\end{align}
where $I_x \in L^Y$ such that $I_x(y) = I(x,y)$ for all $x \in X$ and $y \in Y$.
Thus, if $\mathbf{I}$ is considered in the usual way as a table consisting
of rows corresponding to objects and columns corresponding to attributes,
$I_x$ represents the $\mathbf{L}$-set of attributes of object $x \in X$,
i.e., it represents a single row of the table.
Thus, $||A \Rightarrow B||^*_\mathbf{I}$ is indeed a degree to which
the following condition is true: ``For each object $x \in X$,
if (it is very true that) the object has all the attributes from $A$,
then it has all the attributes from $B$''. If $\Sigma$ satisfies
\begin{align}
  ||A \Rightarrow B||^*_\Sigma &= ||A \Rightarrow B||^*_\mathbf{I}
\end{align}
for all $A,B \in L^Y$, then $\Sigma$ is called complete in $\mathbf{I}$.
In addition, if $\Sigma$ is non-redundant then it is called
a (non-redundant) base of $\mathbf{I}$;
if $\Sigma$ is minimal then it is called a minimal base of $\mathbf{I}$.
In the description of bases, we consider the
well-known concept-forming operators with a linguistic hedge~\cite{BeVy:Fcalh}.
Namely, we use a couple of operators ${}^{\uparrow}\!: L^X \to L^Y$ and 
${}^{\downarrow}\!: L^Y \to L^X$ defined by
\begin{align}
  A^{\uparrow}(y) &=
  \textstyle\bigwedge_{x \in X}\bigl(A(x)^* \rightarrow I(x,y)\bigr),
  \label{eqn:up} \\
  B^{\downarrow}(x) &=
  \textstyle\bigwedge_{y \in Y}\bigl(B(y) \rightarrow I(x,y)\bigr),
  \label{eqn:dn}
\end{align}
for all $A \in L^X$, $B \in L^Y$, $x \in X$, and $y \in Y$.
Importantly, the composition ${}^{\downarrow\uparrow}$ of ${}^{\downarrow}$
and ${}^{\uparrow}$ is an $\mathbf{L}^{\!*}$-closure
operator~\cite{BeFuVy:Fcots} and it follows that $\Sigma$ is
complete in $\mathbf{I}$ if{}f $\mathrm{Mod}^*(\Sigma)$ is the set of all
fixed points of ${}^{\downarrow\uparrow}$, see~\cite{BeVy:ICFCA,BeVy:ADfDwG}.
Therefore, using this fact, we immediately get the following consequence
of Theorem~\ref{th:witnessed}:

\begin{corollary}\label{cor:getbase}
  Let\/ $\mathbf{I} = \langle X,Y,I\rangle$ be a formal\/ $\mathbf{L}$-context
  and let $\Gamma$ be a non-redundant/minimal base of $\mathbf{I}$ which
  consists of formulas with saturated consequents. If $\mathbf{L}$ is finite
  then any irreducible $\Sigma$ such that $\yields{\Gamma}{\Sigma}$
  is a non-redundant/minimal base of\/ $\mathbf{I}$ whose non-redundancy
  is witnessed.
  \qed
\end{corollary}

Therefore, from the existence of non-redundant bases with witnessed
non-redundancy, we can directly derive, using~\cite[Theorem 10]{Vy:Osgaiwnr},
the existence of general systems of pseudo-intents.
Let us recall~\cite{BeVy:Faicnbumis,BeVy:Cnbirdtga,BeVy:ADfDwG} that denoting
\begin{align}
  \mathcal{U} = \{P \in L^Y;\, P \ne P^{\downarrow\uparrow}\},
  \label{eqn:U}
\end{align}
we call $\mathcal{P} \subseteq \mathcal{U}$ a system of pseudo-intents
(of $\mathbf{I}$) whenever for each $P \in \mathcal{U}$, we have
\begin{align}
  P \in \mathcal{P}
  \text{ if{}f }
  ||Q \Rightarrow Q^{\downarrow\uparrow}||^*_P = 1 
  \text{ for any }
  Q \in \mathcal{P}
  \text{ such that }
  Q \ne P.
  \label{eqn:system_P}
\end{align}
If $\mathcal{P}$ exists, then 
\begin{align}
  \Sigma &=
  \{P \Rightarrow P^{\downarrow\uparrow};\, P \in \mathcal{P}\}
  \label{eqn:nred}
\end{align}
is a non-redundant base of $\mathbf{I}$, see~\cite[Theorem~10]{BeVy:Falaitvenb},
and its non-redundancy is witnessed~\cite{Vy:Osgaiwnr}
and, obviously,
each formula in $\Sigma$ has a saturated consequent.
So far, the existence of such systems of pseudo-intents was proved only
for finite $\mathbf{L}$ with globalization,
see~\cite{BeChVy:Ifdwfa,BeVy:Falaitvenb}
but~\cite[Theorem 10]{Vy:Osgaiwnr} shows
that $\Sigma$ is a non-redundant base of $\mathbf{I}$ consisting of
graded attribute implications with saturated consequents such that its
non-redundancy is witnessed if{}f 
\begin{align}
  \mathcal{P} = \{A \in L^Y\!;\,
  A \Rightarrow A^{\downarrow\uparrow} \in \Sigma\}
  \label{eqn:P}
\end{align}
is a system of pseudo-intents of $\mathbf{I}$. Therefore, we conclude:

\begin{corollary}\label{cor:psudos_exist!}
  Let $\mathbf{L}$, $X$, and $Y$ be finite.
  For arbitrary ${}^*$ satisfying \eqref{TS:1}--\eqref{TS:idm},
  any formal\/ $\mathbf{L}$-context\/ $\mathbf{I} = \langle X,Y,I\rangle$
  has at least one system of pseudo-intents which determines a minimal base.
  \qed
\end{corollary}

In addition, it follows from our previous remarks that a system of
pseudo-intents of $\mathbf{I}$ can be computed, in a polynomial time,
from any complete set in~$\mathbf{I}$. Thus, the procedure is tractable.
As we have said before, this is in contrast with approaches to finding
systems of pseudo-intents
directly from $\mathbf{I}$ which are known, even in the classic case,
to be hard, see~\cite{DiSe:OCEP} for the results on complexity 
related to enumerating pseudo-intents.

\begin{remark}
  We have presented the solution of the problem of existence of general
  systems of pseudo-intents using formal $\mathbf{L}$-contexts as structures
  in which we evaluate formulas. In a similar setting, we could do the same
  with other semantic structures which yield the same notion of
  semantic entailment. For instance, in relational similarity-based
  databases~\cite{BeVy:DASFAA}, we can think of bases of similarity-based
  functional dependencies which are true in data tables over domains with
  similarities or their ranked extensions~\cite{BeVy:Qssdlfepc}. In all
  such cases, the important point is that systems of models of such rules
  can be identified with systems of fixed points of $\mathbf{L}^*$-closure
  operators and, therefore, the procedure introduced in our paper, can also
  be applied in these settings where the notion of a system of pseudo-intents
  is present, see, e.g., \cite{BeVy:DASFAA}.
\end{remark}

\section{Proofs and notes}\label{sec:proof}
In this section, we present the proof of Theorem~\ref{th:witnessed} and
present further notes, including an illustrative example, and some experimental
observations which illustrate that in the case of other hedges than
the globalization, non-redundant bases given by systems of pseudo-intents
may be smaller. In order to prove the main assertion of the paper,
we investigate properties of $\yields[]{}{}$. In all the subsequent
lemmas, we assume that $\mathbf{L}$ is a finite residuated lattice
and ${}^*$ satisfies~\eqref{TS:1}--\eqref{TS:idm}.

\begin{lemma}\label{le:equiv}
  If $\yields[]{\Delta}{\Gamma}$, then $\Delta \equiv \Gamma$.
\end{lemma}
\begin{proof}
  Suppose that $\yields[]{\Delta}{\Gamma}$, i.e., 
  there are distinct $A \Rightarrow B \in \Delta$ and $C \Rightarrow D \in \Delta$
  such that $\Gamma$ is of the form~\eqref{eqn:reduct_Gamma}. Thus, $\Gamma$
  results from $\Delta$ by removing $A \Rightarrow B$ and adding
  the formula $A \cup (S(C,A)^* \otimes D) \Rightarrow B$. Trivially,
  for its antecedent $A \cup (S(C,A)^* \otimes D)$, we have
  $A \subseteq A \cup (S(C,A)^* \otimes D)$, i.e.,
  $\Gamma \mathrel{\unlhd} \Delta$. Indeed, for any $M \in L^Y$,
  using the antitony of $\rightarrow$ and $S$ in their first arguments
  together with the isotony of ${}^*$, we get that
  \begin{align*}
    ||A \Rightarrow B||^*_M &=
    S(A,M)^* \rightarrow S(B,M)
    \\
    &\leq S(A \cup (S(C,A)^* \otimes D),M)^* \rightarrow S(B,M)
    \\
    &= ||A \cup (S(C,A)^* \otimes D) \Rightarrow B||^*_M.
  \end{align*}
  Therefore, $\mathrm{Mod}^*(\Delta) \subseteq \mathrm{Mod}^*(\Gamma)$
  and thus $\Gamma \mathrel{\unlhd} \Delta$, i.e., $\Delta$ is stronger
  than $\Gamma$. It remains to show that $\Delta \mathrel{\unlhd} \Gamma$
  holds as well. Take any $M \in \mathrm{Mod}^*(\Gamma)$. Obviously,
  $M \in \mathrm{Mod}^*(\Delta \setminus \{A \Rightarrow B\})$.
  By the adjointness and~\eqref{TS:subd},
  we get $S(C,A)^* \otimes C \subseteq A$.
  Therefore, using properties of residuated lattices and hedges
  together with the fact that $||C \Rightarrow D||^*_M = 1$, we get
  \begin{align*}
    S(A,M)^*
    &\leq S(S(C,A)^* \otimes C,M)^* 
    = (S(C,A)^* \rightarrow S(C,M))^* 
    \\
    &\leq S(C,A)^{**} \rightarrow S(C,M)^*
    = S(C,A)^* \rightarrow S(C,M)^*
    \\
    &\leq S(C,A)^* \rightarrow S(D,M)
    = S(S(C,A)^* \otimes D,M).
  \end{align*}
  Therefore, we may write
  \begin{align*}
    S(A,M)^* &=
    S(A,M) \wedge S(A,M)^*
    \leq S(A,M) \wedge S(S(C,A)^* \otimes D,M)
    \\
    &=
    S(A \cup (S(C,A)^* \otimes D),M)
  \end{align*}
  From the last inequality,
  applying $||A \cup (S(C,A)^* \otimes D) \Rightarrow B||^*_M=1$, we get
  \begin{align*}
    S(A,M)^* \leq
    S(A \cup (S(C,A)^* \otimes D),M)^*
    \leq 
    S(B,M),
  \end{align*}
  which proves that $||A \Rightarrow B||^*_M = 1$,
  i.e., $M \in \mathrm{Mod}^*(\Delta)$.
  Altogether, $\Delta \equiv \Gamma$.
\end{proof}

\begin{remark}
  Readers familiar with Armstrong-style
  axiomatic systems~\cite{Arm:Dsdbr} of the logic of graded
  attribute implications parameterized by linguistic
  hedges, see~\cite{BeCoEnMoVy:Aermdsod,BeVy:MRAP,BeVy:Lga,UrVy:Dddosbd},
  may easily see that Lemma~\ref{le:equiv}
  can be proved by arguments about provability
  and utilizing the completeness theorems.
  Indeed, $\Gamma \mathrel{\unlhd} \Delta$ is a consequence of the fact
  that the rule of weakening
  (also known as augmentation~\cite{BeCoEnMoVy:Aermdsod,Mai:TRD}:
  from $A \Rightarrow B$ we infer
  $A \cup E \Rightarrow B$ for $E = S(C,A)^* \otimes D$)
  is a derivable deduction rule in the logic.
  In addition, $\Delta \mathrel{\unlhd} \Gamma$ follows by the rules
  of multiplication (from $C \Rightarrow D$ we infer
  $a^*{\otimes}C \Rightarrow a^*{\otimes}D$ for $a = S(C,A)$) and
  cut (also known as pseudo-transitivity~\cite{Hol,Mai:TRD}:
  from $S(C,A)^*{\otimes}C \Rightarrow S(C,A)^*{\otimes}D$
  and $A \cup (S(C,A)^* \otimes D) \Rightarrow B$
  we infer $A \cup (S(C,A)^*{\otimes}C) \Rightarrow B$
  which equals to $A \Rightarrow B$).
\end{remark}

\begin{lemma}\label{le:satcons}
  If $\Delta$ has saturated consequents and $\yields[]{\Delta}{\Gamma}$,
  then so has $\Gamma$.
\end{lemma}
\begin{proof}
  Assume that $A \Rightarrow B \in \Delta$ and $C \Rightarrow D \in \Delta$
  are the formulas from~\eqref{eqn:reduct_Gamma}. Since $\Delta$ consists
  of formulas with saturated consequents,
  then $A \Rightarrow B$ and $C \Rightarrow D$ are
  in fact in the form 
  $A \Rightarrow [A]^*_\Delta$ and $C \Rightarrow [C]^*_\Delta$,
  respectively. Furthermore, using Lemma~\ref{le:equiv},
  $\Delta \equiv \Gamma$, i.e., $[E]^*_\Delta = [E]^*_\Gamma$
  for any $E \in L^Y$.
  Thus, each formula in $\Gamma \setminus
  \{A \cup (S(C,A)^* \otimes [C]^*_\Gamma) \Rightarrow [A]^*_\Gamma\}$
  has a saturated consequent. Therefore, it remains to check this fact
  for the formula
  $A \cup (S(C,A)^* \otimes [C]^*_\Gamma) \Rightarrow [A]^*_\Gamma$.
  By~\eqref{eqn:cl1}, we get $A \subseteq [A]^*_\Gamma$.
  Applying~\eqref{eqn:cl2}, we get $S(C,A)^* \leq S([C]^*_\Gamma,[A]^*_\Gamma)$
  and so the adjointness property gives
  $S(C,A)^* \otimes [C]^*_\Gamma \subseteq [A]^*_\Gamma$.
  Therefore, using \eqref{eqn:cl2} and \eqref{eqn:cl3},
  the previous inclusions yield
  \begin{align}
    \bigl[A \cup (S(C,A)^* \otimes [C]^*_\Gamma)\bigr]^*_\Gamma
    \subseteq [A]^*_\Gamma.
    \label{eqn:satcons}
  \end{align}
  The converse inclusion to that in~\eqref{eqn:satcons} follows directly
  from~\eqref{eqn:cl2}. As a consequence,
  $A \cup (S(C,A)^* \otimes [C]^*_\Gamma) \Rightarrow [A]^*_\Gamma$
  has a saturated consequent.
\end{proof}

\begin{lemma}\label{le:nonred}
  If $\Delta$ is non-redundant and $\yields[]{\Delta}{\Gamma}$,
  then so is $\Gamma$.
\end{lemma}
\begin{proof}
  Again, we assume that $\Gamma$ is a theory of
  the form~\eqref{eqn:reduct_Gamma} and $\Delta$ is non-redundant.
  First, note that $A \cup (S(C,A)^* \otimes D) \Rightarrow B \not\in \Delta$
  because otherwise $A \cup (S(C,A)^* \otimes D) \Rightarrow B$ would be
  redundant in $\Delta$ because $A \Rightarrow B \in \Delta$---a contradiction.
  Now, take arbitrary $E \Rightarrow F \in
  \Gamma \setminus \{A \cup (S(C,A)^* \otimes D) \Rightarrow B\}$.
  The fact that $E \Rightarrow F$ is non-redundant in $\Delta$ means that
  $||E \Rightarrow F||^*_{\Delta \setminus \{E \Rightarrow F\}} < 1$.
  By similar arguments as in the proof of Lemma~\ref{le:equiv},
  it follows that $\Gamma \setminus \{E \Rightarrow F\} \mathrel{\unlhd}
  \Delta \setminus \{E \Rightarrow F\}$.
  Therefore, we have 
  \begin{align*}
    ||E \Rightarrow F||^*_{\Gamma \setminus \{E \Rightarrow F\}}
    \leq
    ||E \Rightarrow F||^*_{\Delta \setminus \{E \Rightarrow F\}}
    < 1
  \end{align*}
  which proves that $E \Rightarrow F$ is non-redundant in $\Gamma$.
  Thus, it remains to prove that $A \cup (S(C,A)^* \otimes D) \Rightarrow B$
  is non-redundant in $\Gamma$. By contradiction, let
  \begin{align*}
    ||A \cup (S(C,A)^* \otimes D) \Rightarrow B||^*_{\Gamma
      \setminus \{A \cup (S(C,A)^* \otimes D) \Rightarrow B\}} = 1
  \end{align*}
  Since $\Gamma$ and $\Delta$ differ only in the formula
  $A \Rightarrow B \in \Delta$ (and $A \cup (S(C,A)^* \otimes D)
  \Rightarrow B \in \Gamma$), we immediately get
  \begin{align*}
    ||A \cup (S(C,A)^* \otimes D) \Rightarrow B||^*_{\Delta
      \setminus \{A \Rightarrow B\}} = 1.
  \end{align*}
  Therefore, for any
  $M \in \mathrm{Mod}^*(\Delta \setminus \{A \Rightarrow B\})$, the previous
  fact yields $M \in \mathrm{Mod}^*(\Gamma)$. Since $\Gamma \equiv \Delta$,
  we get that $M \in \mathrm{Mod}^*(\Delta)$, i.e., 
  $\mathrm{Mod}^*(\Delta \setminus \{A \Rightarrow B\}) \subseteq 
  \mathrm{Mod}^*(\Delta)$ which contradicts the fact that $\Delta$
  is non-redundant.
\end{proof}

\begin{lemma}\label{le:witness}
  Let $\Sigma$ be a non-redundant irreducible theory.
  Then the non-redundancy of $\Sigma$ is witnessed.
\end{lemma}
\begin{proof}
  Take any $A \Rightarrow B \in \Sigma$. It suffices to show that
  $A \in \mathrm{Mod}^*(\Sigma \setminus \{A \Rightarrow B\})$.
  Take any $C \Rightarrow D \in \Sigma \setminus \{A \Rightarrow B\}$,
  i.e., $A \Rightarrow B$ and $C \Rightarrow D$ are distinct formulas.
  Now, observe that $A \cup (S(C,A)^* \otimes D) \Rightarrow B \not\in \Sigma$
  because otherwise the non-redundancy of $\Sigma$ would be violated owing to
  $A \Rightarrow B \in \Sigma$.
  Since $\Sigma$ is irreducible, we must have $S(C,A)^* \otimes D \subseteq A$
  which by adjointness yields $S(C,A)^* \leq S(D,A)$,
  i.e., $A \in \mathrm{Mod}^*(\{C \Rightarrow D\})$.
  As a consequence,
  $A \in \mathrm{Mod}^*(\Sigma \setminus \{A \Rightarrow B\})$,
  which finishes the proof.
\end{proof}

\begin{proof}[Proof of Theorem~\ref{th:witnessed}]
  The proof results by putting together the observations in
  the previous lemmas. Because of the finiteness of $\mathbf{L}$ and $Y$,
  the existence of irreducible $\Sigma$ such that $\yields{\Gamma}{\Sigma}$
  is ensured. Then, by induction, Lemma~\ref{le:equiv}
  yields $\Sigma \equiv \Gamma$, Lemma~\ref{le:nonred} yields
  that $\Sigma$ is non-redundant, Lemma~\ref{le:satcons} yields
  that $\Sigma$ has saturated consequents. Finally, Lemma~\ref{le:witness}
  shows that the non-redundancy of $\Sigma$ is witnessed.
  This proves Theorem~\ref{th:witnessed} for $\Gamma$ being non-redundant.
  If $\Gamma$ is in addition minimal, it is easily seen that $\Sigma$
  is minimal as well---this follows directly by~the form
  of~\eqref{eqn:reduct_Gamma}.
\end{proof}

We conclude this section by illustrative examples and some
experimental observations on sizes of non-redundant bases
determined from data using different hedges. In the examples we
use the usual notation for writing $\mathbf{L}$-sets on finite
universe sets, namely, $\{{}^{a_1\!}/y_1,\ldots,{}^{a_n\!}/y_n\}$
denotes an $\mathbf{L}$-set $A$ in $Y = \{y_1,\ldots,y_n\}$ such
that $A(y_i) = a_i$ for all $i=1,\ldots,n$.
Optionally, we omit ${}^{a_i\!}/y_i$ whenever $a_i = 0$ and write
just $y_i$ instead of ${}^{a_i\!}/y_i$ whenever $a_i = 1$.

\begin{example}\label{ex:osgai_now_works}
  Suppose that $\mathbf{L}$ is a residuated lattice with $L = \{0, 0.5, 1\}$
  and $\otimes = \wedge$ (i.e., $\mathbf{L}$ is a the three-element
  G\"odel chain~\cite{Bel:FRS,Got:Mfl,KMP:TN}) and let ${}^*$ be
  the identity. Furthermore, we take theory $\Gamma$
  from~\cite[Example~4]{Vy:Osgaiwnr}:
  \begin{align*}
    \Gamma &= \{
    \{{}^{0.5\!}/p\} \Rightarrow \{{}^{0.5\!}/p,{}^{0.5\!}/q,r\},
    \{p\} \Rightarrow \{p,q,r\}\}.
  \end{align*}
  For this theory on $Y = \{p,q,r\}$ and the considered $\mathbf{L}$
  and ${}^*$, \cite[Example~4]{Vy:Osgaiwnr} shows that $\Sigma$ given by
  \begin{align*}
    \Sigma &=
    \bigl\{[A]^*_{\Gamma \setminus \{A \Rightarrow [A]^*_\Gamma\}}
    \Rightarrow
    [A]^*_\Gamma;\, A \Rightarrow [A]^*_\Gamma \in \Gamma\bigr\}.
  \end{align*}
  which equals to
  \begin{align*}
    \Sigma &= \{
    \{{}^{0.5\!}/p,{}^{0.5\!}/q,{}^{0.5\!}/r\} \Rightarrow
    \{{}^{0.5\!}/p,{}^{0.5\!}/q,r\},
    \{p,{}^{0.5\!}/q,r\} \Rightarrow \{p,q,r\}\}
  \end{align*}
  is not equivalent to $\Gamma$. This is a particular example of a theory
  which cannot be converted into an equivalent one with witnessed
  non-redundancy using the method described in~\cite{Vy:Osgaiwnr}.
  Using the method shown in the present paper, for $\Gamma$ we may
  find an irreducible theory $\Sigma'$ such that $\yields{\Gamma}{\Sigma'}$.
  Actually, this can be done in a single step:
  For $\{{}^{0.5\!}/p\} \Rightarrow \{{}^{0.5\!}/p,{}^{0.5\!}/q,r\} \in \Gamma$
  and $\{p\} \Rightarrow \{p,q,r\} \in \Gamma$, we can easily see that
  \begin{align*}
    S(\{p\},\{{}^{0.5\!}/p\})^* \otimes \{p,q,r\} =
    0.5 \otimes \{p,q,r\} = \{{}^{0.5\!}/p,{}^{0.5\!}/q,{}^{0.5\!}/r\},
  \end{align*}
  i.e., $\Gamma$ can be reduced to
  \begin{align*}
    \Gamma' &= \{
    \{{}^{0.5\!}/p,{}^{0.5\!}/q,{}^{0.5\!}/r\} \Rightarrow
    \{{}^{0.5\!}/p,{}^{0.5\!}/q,r\},
    \{p\} \Rightarrow \{p,q,r\}\}
  \end{align*}
  and $\Gamma'$ is already irreducible. Thus, Theorem~\ref{th:witnessed}
  yields that $\Gamma'$ is equivalent to $\Gamma$, it is non-redundant,
  has saturated consequents, and its non-redundancy is witnessed. This
  example demonstrates that the outcome of the algorithm studied in the
  previous paper~\cite{Vy:Osgaiwnr} is unnecessary weak in case of hedges
  other than the globalization.
\end{example}

\begin{example}\label{ex:dense}
  \def\AB{A \Rightarrow B \text{ is }}%
  \def\CD{C \Rightarrow D \text{ is }}%
  \def\RPLC#1{\text{replace } A \Rightarrow B \text{ in } \Gamma_{#1} \text{ by }}%
  In this example, we show a transformation of a non-redundant theory
  with saturated consequents into an equivalent theory with witnessed
  non-redundancy which takes more elementary steps. We assume that
  $\mathbf{L}$ is an equidistant five-element \L ukasiewicz chain,
  i.e., $L = \{0,0.25, 0.5, 0.75, 1\}$ and the adjoint operations
  $\otimes$ and $\rightarrow$ are defined as $a \otimes b = \max (0, a+b-1)$
  and $a \rightarrow b = \min (1, 1-a+b)$, respectively. Let
  \begin{align*}
    \Gamma_0 = \{&
    \{{}^{0.25\!}/r,{}^{0.25\!}/s\} \Rightarrow
    \{{}^{0.75\!}/p,{}^{0.75\!}/q,{}^{0.75\!}/r,{}^{0.25\!}/s\}, \\
    &\{{}^{0.5\!}/p,{}^{0.25\!}/q,{}^{0.5\!}/s\} \Rightarrow
    \{{}^{0.75\!}/p,{}^{0.75\!}/q,r,{}^{0.5\!}/s\}, \\
    &\{{}^{0.25\!}/p,s\} \Rightarrow
    \{{}^{0.75\!}/p,q,r,s\}\}.
  \end{align*}
  It can be easily checked that $\Gamma_0$ is non-redundant and
  has saturated consequents. Also, it is evident that the non-redundancy of
  $\Gamma_0$ is not witnessed because it is not irreducible. Indeed, there is
  $\Gamma_1$ such that $\yields[]{\Gamma_0}{\Gamma_1}$. Namely, we can take
  the formulas $A \Rightarrow B$ and $C \Rightarrow D$ which appear
  in~\eqref{eqn:reduct_Gamma} as follows:
  \begin{flalign*}
    &\AB\{{}^{0.25\!}/r,{}^{0.25\!}/s\} \Rightarrow
    \{{}^{0.75\!}/p,{}^{0.75\!}/q,{}^{0.75\!}/r,{}^{0.25\!}/s\}, && \\
    &\CD\{{}^{0.5\!}/p,{}^{0.25\!}/q,{}^{0.5\!}/s\} \Rightarrow
    \{{}^{0.75\!}/p,{}^{0.75\!}/q,r,{}^{0.5\!}/s\},\\
    &S(C,A)^*\otimes D =
    0.5\otimes\{{}^{0.75\!}/p,{}^{0.75\!}/q,r,{}^{0.5\!}/s\} =
    \{{}^{0.25\!}/p,{}^{0.25\!}/q,{}^{0.5\!}/r\},\\
    &\RPLC{0} \{{}^{0.25\!}/p,{}^{0.25\!}/q,{}^{0.5\!}/r,{}^{0.25\!}/s\} \Rightarrow
    \{{}^{0.75\!}/p,{}^{0.75\!}/q,{}^{0.75\!}/r,{}^{0.25\!}/s\}.
  \end{flalign*}
  Thus, $\Gamma_1 = (\Gamma_0 \setminus \{A \Rightarrow B\}) \cup 
  \{\{{}^{0.25\!}/p,{}^{0.25\!}/q,{}^{0.5\!}/r,{}^{0.25\!}/s\} \Rightarrow B\}$.
  For the new formula, we can repeat the procedure with the same $C \Rightarrow D$.
  That is,
  \begin{flalign*}
    &\AB\{{}^{0.25\!}/p,{}^{0.25\!}/q,{}^{0.5\!}/r,{}^{0.25\!}/s\} \Rightarrow
    \{{}^{0.75\!}/p,{}^{0.75\!}/q,{}^{0.75\!}/r,{}^{0.25\!}/s\}, && \\
    &\CD\{{}^{0.5\!}/p,{}^{0.25\!}/q,{}^{0.5\!}/s\} \Rightarrow
    \{{}^{0.75\!}/p,{}^{0.75\!}/q,r,{}^{0.5\!}/s\},\\
    &S(C,A)^*\otimes D =
    0.75\otimes\{{}^{0.75\!}/p,{}^{0.75\!}/q,r,{}^{0.5\!}/s\} =
    \{{}^{0.5\!}/p,{}^{0.5\!}/q,{}^{0.75\!}/r,{}^{0.25\!}/s\},\\
    &\RPLC{1} \{{}^{0.5\!}/p,{}^{0.5\!}/q,{}^{0.75\!}/r,{}^{0.25\!}/s\} \Rightarrow
    \{{}^{0.75\!}/p,{}^{0.75\!}/q,{}^{0.75\!}/r,{}^{0.25\!}/s\},
  \end{flalign*}
  i.e., $\Gamma_2 = (\Gamma_1 \setminus \{A \Rightarrow B\}) \cup
  \{\{{}^{0.5\!}/p,{}^{0.5\!}/q,{}^{0.75\!}/r,{}^{0.25\!}/s\} \Rightarrow B$.
  In the next step, we update the antecedent of the formula which so far played
  the role of $C \Rightarrow D$:
  \begin{flalign*}
    &\AB\{{}^{0.5\!}/p,{}^{0.25\!}/q,{}^{0.5\!}/s\} \Rightarrow
    \{{}^{0.75\!}/p,{}^{0.75\!}/q,r,{}^{0.5\!}/s\}, && \\
    &\CD\{{}^{0.25\!}/p,s\} \Rightarrow \{{}^{0.75\!}/p,q,r,s\},\\
    &S(C,A)^*\otimes D = 0.5\otimes\{{}^{0.75\!}/p,q,r,s\} =
    \{{}^{0.25\!}/p,{}^{0.5\!}/q,{}^{0.5\!}/r,{}^{0.5\!}/s\},\\
    &\RPLC{2} \{{}^{0.5\!}/p,{}^{0.5\!}/q,{}^{0.5\!}/r,{}^{0.5\!}/s\} \Rightarrow
    \{{}^{0.75\!}/p,{}^{0.75\!}/q,r,{}^{0.5\!}/s\}.
  \end{flalign*}
  Thus, $\Gamma_3 = (\Gamma_2 \setminus \{A \Rightarrow B\}) \cup
  \{\{{}^{0.5\!}/p,{}^{0.5\!}/q,{}^{0.5\!}/r,{}^{0.5\!}/s\} \Rightarrow B\}$.
  In the next step, we flip the roles of $A \Rightarrow B$ and $C \Rightarrow D$
  from the previous step and make an update of the antecedent of
  $\{{}^{0.25\!}/p,s\} \Rightarrow \{{}^{0.75\!}/p,q,r,s\}$:
  \begin{flalign*}
    &\AB\{{}^{0.25\!}/p,s\} \Rightarrow \{{}^{0.75\!}/p,q,r,s\}, && \\
    &\CD\{{}^{0.5\!}/p,{}^{0.5\!}/q,{}^{0.5\!}/r,{}^{0.5\!}/s\} \Rightarrow
    \{{}^{0.75\!}/p,{}^{0.75\!}/q,r,{}^{0.5\!}/s\},\\
    &S(C,A)^*\otimes D =
    0.5\otimes\{{}^{0.75\!}/p,{}^{0.75\!}/q,r,{}^{0.5\!}/s\} =
    \{{}^{0.25\!}/p,{}^{0.25\!}/q,{}^{0.5\!}/r\},\\
    &\RPLC{3} \{{}^{0.25\!}/p,{}^{0.25\!}/q,{}^{0.5\!}/r,s\} \Rightarrow
    \{{}^{0.75\!}/p,q,r,s\},
  \end{flalign*}
  i.e., $\Gamma_4 = (\Gamma_3 \setminus \{A \Rightarrow B\}) \cup
  \{\{{}^{0.25\!}/p,{}^{0.25\!}/q,{}^{0.5\!}/r,s\} \Rightarrow B\}$.
  In the last two steps, we use the same formulas for update, i.e.,
  \begin{flalign*}
    &\AB\{{}^{0.25\!}/p,{}^{0.25\!}/q,{}^{0.5\!}/r,s\} \Rightarrow
    \{{}^{0.75\!}/p,q,r,s\}, && \\
    &\CD\{{}^{0.5\!}/p,{}^{0.5\!}/q,{}^{0.5\!}/r,{}^{0.5\!}/s\} \Rightarrow
    \{{}^{0.75\!}/p,{}^{0.75\!}/q,r,{}^{0.5\!}/s\},\\
    &S(C,A)^*\otimes D =
    0.75\otimes\{{}^{0.75\!}/p,{}^{0.75\!}/q,r,{}^{0.5\!}/s\} =
    \{{}^{0.5\!}/p,{}^{0.5\!}/q,{}^{0.75\!}/r,{}^{0.25\!}/s\},\\
    &\RPLC{4} \{{}^{0.5\!}/p,{}^{0.5\!}/q,{}^{0.75\!}/r,s\} \Rightarrow
    \{{}^{0.75\!}/p,q,r,s\}.
  \end{flalign*}
  Thus, $\Gamma_5 = (\Gamma_4 \setminus \{A \Rightarrow B\}) \cup
  \{\{{}^{0.5\!}/p,{}^{0.5\!}/q,{}^{0.75\!}/r,s\} \Rightarrow B\}$.
  Finally, we can see that $C$ is now fully included in $A$, i.e.,
  \begin{flalign*}
    &\AB\{{}^{0.5\!}/p,{}^{0.5\!}/q,{}^{0.75\!}/r,s\} \Rightarrow
    \{{}^{0.75\!}/p,q,r,s\}, && \\
    &\CD\{{}^{0.5\!}/p,{}^{0.5\!}/q,{}^{0.5\!}/r,{}^{0.5\!}/s\} \Rightarrow
    \{{}^{0.75\!}/p,{}^{0.75\!}/q,r,{}^{0.5\!}/s\},\\
    &S(C,A)^*\otimes D =
    1.0\otimes\{{}^{0.75\!}/p,{}^{0.75\!}/q,r,{}^{0.5\!}/s\} =
    \{{}^{0.75\!}/p,{}^{0.75\!}/q,r,{}^{0.5\!}/s\},\\
    &\RPLC{5} \{{}^{0.75\!}/p,{}^{0.75\!}/q,r,s\} \Rightarrow
    \{{}^{0.75\!}/p,q,r,s\}.
  \end{flalign*}
  Now, the last computed theory $\Gamma_5$:
  \begin{align*}
    \Gamma_5 = \{
    &\{\{{}^{0.75\!}/p,{}^{0.75\!}/q,r,s\} \Rightarrow
    \{{}^{0.75\!}/p,q,r,s\}, \\
    &\{{}^{0.5\!}/p,{}^{0.5\!}/q,{}^{0.5\!}/r,{}^{0.5\!}/s\} \Rightarrow
    \{{}^{0.75\!}/p,{}^{0.75\!}/q,r,{}^{0.5\!}/s\}, \\
    &\{{}^{0.5\!}/p,{}^{0.5\!}/q,{}^{0.75\!}/r,{}^{0.25\!}/s\} \Rightarrow
    \{{}^{0.75\!}/p,{}^{0.75\!}/q,{}^{0.75\!}/r,{}^{0.25\!}/s\}\}
    \}
  \end{align*}
  is irreducible, i.e., its non-redundancy is witnessed.
\end{example}

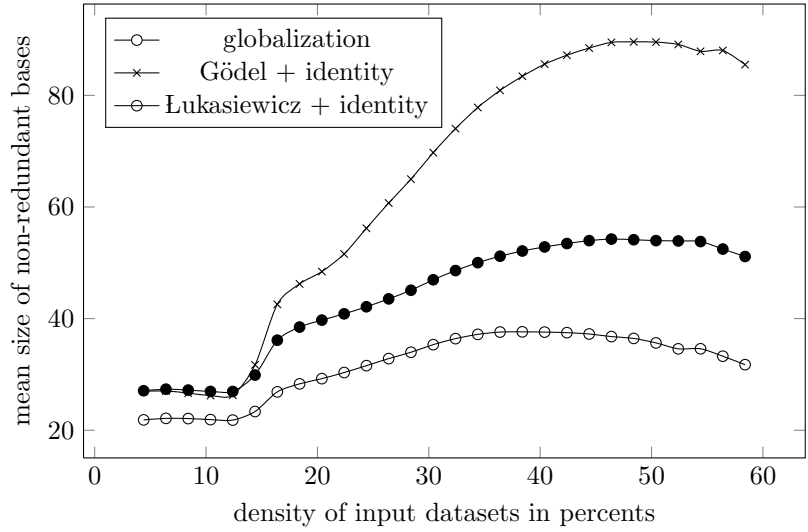
\begin{figure}[p]
  \centering
  \begin{tikzpicture}
    \begin{axis}[
      legend pos={north west},
      width=4.4in,
      height=3in,
      xlabel={density of input datasets in percents},
      ylabel={mean size of non-redundant bases}]
      \addplot[smooth,mark=*] plot coordinates {
        (4.4, 27.08333)
        (6.4, 27.36377)
        (8.4, 27.19948)
        (10.4, 26.95884)
        (12.4, 26.95452)
        (14.4, 29.89208)
        (16.4, 36.13858)
        (18.4, 38.48821)
        (20.4, 39.72961)
        (22.4, 40.85608)
        (24.4, 42.12439)
        (26.4, 43.53862)
        (28.4, 45.08374)
        (30.4, 46.96050)
        (32.4, 48.60913)
        (34.4, 50.02156)
        (36.4, 51.18034)
        (38.4, 52.11296)
        (40.4, 52.84335)
        (42.4, 53.44925)
        (44.4, 53.96968)
        (46.4, 54.25108)
        (48.4, 54.13446)
        (50.4, 53.97172)
        (52.4, 53.91469)
        (54.4, 53.80597)
        (56.4, 52.46154)
        (58.4, 51.12500)};
    \addlegendentry{globalization}
      \addplot[smooth,mark=x] plot coordinates {
        (4.4, 26.97222)
        (6.4, 27.04280)
        (8.4, 26.65366)
        (10.4, 26.21883)
        (12.4, 26.27337)
        (14.4, 31.71878)
        (16.4, 42.56228)
        (18.4, 46.21883)
        (20.4, 48.42900)
        (22.4, 51.58438)
        (24.4, 56.18126)
        (26.4, 60.71172)
        (28.4, 64.98307)
        (30.4, 69.73750)
        (32.4, 74.03585)
        (34.4, 77.82301)
        (36.4, 80.88539)
        (38.4, 83.41415)
        (40.4, 85.58308)
        (42.4, 87.18458)
        (44.4, 88.46551)
        (46.4, 89.49138)
        (48.4, 89.57669)
        (50.4, 89.53737)
        (52.4, 89.11848)
        (54.4, 87.88060)
        (56.4, 88.03846)
        (58.4, 85.50000)};
      \addlegendentry{G\"odel + identity}
      \addplot[smooth,mark=o] plot coordinates {
        (4.4, 21.86111)
        (6.4, 22.13124)
        (8.4, 22.09540)
        (10.4, 21.90983)
        (12.4, 21.82262)
        (14.4, 23.35893)
        (16.4, 26.87926)
        (18.4, 28.31129)
        (20.4, 29.24769)
        (22.4, 30.34321)
        (24.4, 31.57411)
        (26.4, 32.83738)
        (28.4, 33.97340)
        (30.4, 35.33632)
        (32.4, 36.41410)
        (34.4, 37.17894)
        (36.4, 37.59591)
        (38.4, 37.63145)
        (40.4, 37.58626)
        (42.4, 37.48308)
        (44.4, 37.24825)
        (46.4, 36.78017)
        (48.4, 36.45319)
        (50.4, 35.64848)
        (52.4, 34.58768)
        (54.4, 34.59701)
        (56.4, 33.26923)
        (58.4, 31.75000)};
      \addlegendentry{\L ukasiewicz + identity}
    \end{axis}
  \end{tikzpicture}
  \caption{Mean sizes of non-redundant bases of formal $\mathbf{L}$-contexts with
    $5$ objects and $10$ attributes computed for structures of degrees with
    $|L| = 11$ using globalization as the hedge, and using G\"odel and
    \L ukasiewicz operations with identity as the hedge.}
  \label{fig:sizes}
\end{figure}

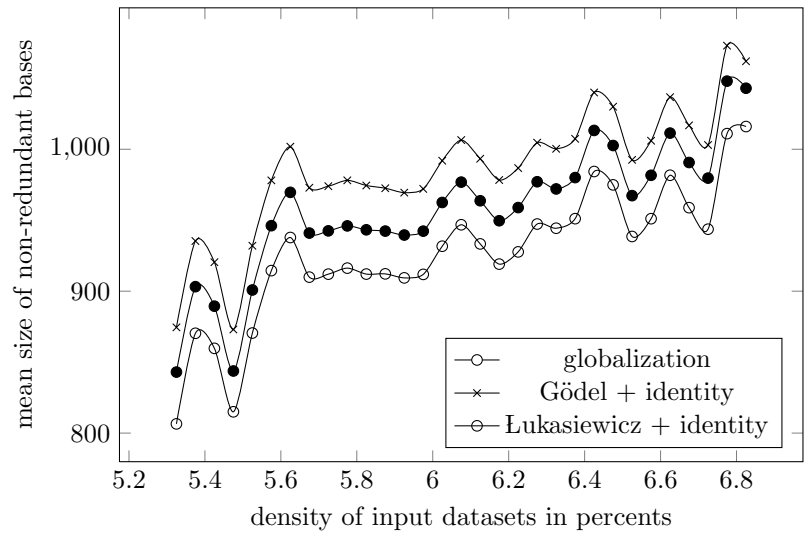
\begin{figure}[p]
  \centering
  \begin{tikzpicture}
    \begin{axis}[
      legend pos={south east},
      width=4.2in,
      height=3in,
      xlabel={density of input datasets in percents},
      ylabel={mean size of non-redundant bases}]
      \addplot[smooth,mark=*] plot coordinates {
        (5.325, 843.0000)
        (5.375, 903.1667)
        (5.425, 889.4286)
        (5.475, 843.7500)
        (5.525, 900.8333)
        (5.575, 946.1429)
        (5.625, 969.6364)
        (5.675, 940.9444)
        (5.725, 942.5000)
        (5.775, 946.0000)
        (5.825, 943.2381)
        (5.875, 942.3333)
        (5.925, 939.5581)
        (5.975, 942.3000)
        (6.025, 962.5106)
        (6.075, 976.8444)
        (6.125, 963.7059)
        (6.175, 949.6182)
        (6.225, 958.9048)
        (6.275, 977.0952)
        (6.325, 972.0357)
        (6.375, 980.0455)
        (6.425,1013.2632)
        (6.475,1002.6364)
        (6.525, 967.3077)
        (6.575, 981.6667)
        (6.625,1011.3846)
        (6.675, 990.6667)
        (6.725, 979.6667)
        (6.775,1048.0000)
        (6.825,1043.0000)};
      \addlegendentry{globalization}
      \addplot[smooth,mark=x] plot coordinates {
        (5.325, 874.5000)
        (5.375, 935.3333)
        (5.425, 920.4286)
        (5.475, 873.0000)
        (5.525, 932.0000)
        (5.575, 978.0714)
        (5.625,1001.8182)
        (5.675, 972.8889)
        (5.725, 974.0000)
        (5.775, 978.0952)
        (5.825, 974.4762)
        (5.875, 972.5333)
        (5.925, 969.4419)
        (5.975, 972.0750)
        (6.025, 991.9362)
        (6.075,1006.4889)
        (6.125, 993.3137)
        (6.175, 978.3273)
        (6.225, 986.7381)
        (6.275,1004.6429)
        (6.325,1000.3214)
        (6.375,1007.3636)
        (6.425,1040.0000)
        (6.475,1030.0000)
        (6.525, 992.5385)
        (6.575,1006.0000)
        (6.625,1036.6923)
        (6.675,1016.8333)
        (6.725,1003.0000)
        (6.775,1073.0000)
        (6.825,1062.0000)};
      \addlegendentry{G\"odel + identity}
      \addplot[smooth,mark=o] plot coordinates {
        (5.325, 806.5000)
        (5.375, 870.3333)
        (5.425, 859.7143)
        (5.475, 815.0000)
        (5.525, 870.5000)
        (5.575, 914.5000)
        (5.625, 937.8182)
        (5.675, 909.8889)
        (5.725, 912.0000)
        (5.775, 916.2381)
        (5.825, 912.0476)
        (5.875, 912.2000)
        (5.925, 909.3721)
        (5.975, 911.8250)
        (6.025, 931.6809)
        (6.075, 946.8000)
        (6.125, 933.3137)
        (6.175, 919.0727)
        (6.225, 927.6905)
        (6.275, 947.3095)
        (6.325, 944.3571)
        (6.375, 951.0909)
        (6.425, 984.3158)
        (6.475, 974.9091)
        (6.525, 938.5385)
        (6.575, 951.0667)
        (6.625, 981.6923)
        (6.675, 958.8333)
        (6.725, 943.6667)
        (6.775,1011.0000)
        (6.825,1016.0000)};
      \addlegendentry{\L ukasiewicz + identity}
    \end{axis}
  \end{tikzpicture}
  \caption{Mean sizes of non-redundant bases of sparse formal $\mathbf{L}$-contexts
    with $20$ objects and $50$ attributes for $|L| = 5$.}
  \label{fig:sparse}
\end{figure}

\begin{example}\label{ex:experiments}
  In the last example, we present results of some experimental observations
  on sizes of non-redundant bases with witnessed non-redundancy computed for
  randomly generated formal $\mathbf{L}$-contexts. Fig.~\ref{fig:sizes} shows
  the mean sizes of bases computed from $\mathbf{L}$-contexts with $5$ object and
  $10$ attributes. The set of truth degrees used in the experiment is
  an $11$-element equidistant subchain of $[0,1]$.
  The $x$-axis in Fig.~\ref{fig:sizes} indicates the density of input data
  which if for $\mathbf{I} = \langle X,Y,I\rangle$ expressed in percents by
  \begin{align*}
    \cfrac{%
      \textstyle\sum_{x \in X}\sum_{y \in Y}I(x,y)}{%
      |X| \cdot |Y|} \cdot 100.
  \end{align*}
  The $y$-axis indicates the mean number of formulas in a non-redundant base
  with witnessed non-redundancy. For each randomly generated $\mathbf{L}$-contexts,
  we have computed a base using the globalization
  (which completely supersedes the role of $\otimes$ and
  $\rightarrow$) and identity as the hedge. In the latter case, we have used
  the standard \L ukasiewicz and G\"odel operations on $[0,1]$. The total number
  of generated formal $\mathbf{L}$-contexts for the experiment is $58,300$.
  As it turns out, the \L ukasiewicz operations with identity as the hedge
  yield smaller bases which is most evident if the density of the
  input data is about 50\,\%.
  Interestingly, in most cases, the G\"odel operations with identity yield
  greater bases than in the case of globalization with the exception of
  sparse data sets but this may be just a deviation from a typical behavior
  which is due to the small size of the data. In Fig.~\ref{fig:sparse}, we
  have included the results of a similar experiment on randomly generated
  data with $20$ objects, $50$ attributes, and using $5$ truth degrees.
  Here, even if the data is sparse, the sizes of bases for 
  the G\"odel operations with identity are bigger than that for globalization.
  Even in this case, the \L ukasiewicz operations with identity
  seem to produce the smallest bases. Let us note that in order to produce
  Fig.~\ref{fig:sparse}, we have used $300$ randomly
  generated $\mathbf{L}$-contexts.
\end{example}

\section{Conclusion}\label{sec:concl}
Our paper deals with transformations of sets of if-then rules describing
dependencies between graded attributes which generalize the ordinary attribute
implications in a setting where presence of attributes is expressed by degrees. The
transformations we deal with allow us to transform a set of such dependencies into
an equivalent one which is non-redundant and satisfies a stronger condition of
witnessed non-redundancy. The witnessed non-redundancy may be seen as a desirable
property because the non-redundancy of every rule in the set can be directly
observed from its antecedent. Compared to our previous results, the proposed
algorithm works for arbitrary idempotent truth-stressing linguistic hedge which
may serve as a parameter of interpretation of the rules. As one of the consequences
of the theoretical result, we obtained that every finite object-attribute data
table with graded attributes (formal $\mathbf{L}$-context) admits at least one
system of pseudo-intents provided that the utilized structure of degrees is finite
and, in addition, it admits a system of pseudo-intents which determines
a minimal base of the data table. By this, we have closed one of the open problems
listed in~\cite{Kw:Open2006}.
We have also presented some initial experimental
evidence that bases of if-then rules computed using different hedges than the
globalization can be smaller and thus more interesting for users. Clearly, in the
graded setting, the topics related to non-redundancy and minimality of bases
are considerably more involved than in the classic setting and further investigation
focused on theory, algorithms, and experiments is needed.

\subsubsection*{Acknowledgment}
Supported by grant no. \verb|P202/14-11585S| of the Czech Science Foundation.


\footnotesize
\bibliographystyle{amsplain}
\bibliography{csgaiwnr}

\end{document}